\documentclass[runningheads]{llncs}
\usepackage[T1]{fontenc}
\usepackage{amsmath,amssymb}
\usepackage{amsfonts}
\usepackage{algorithm}
\usepackage{algorithmicx}
\usepackage{algpseudocode}
\usepackage{graphicx}
\usepackage{float}
\usepackage{tikz}
\usepackage{rotating}
\usepackage{booktabs}
\usepackage{pdflscape}
\usepackage{booktabs}
\usepackage{multirow}

\usepackage{graphicx}
\begin{document}
\title{Achieving Fair PCA Using Joint Eigenvalue Decomposition}
\titlerunning{JEVD-PCA}
%
\author{Paper ID: 701}
\author{Vidhi Rathore \and
Naresh Manwani}
\authorrunning{Rathore and Manwani}
%
\institute{Machine Learning Lab, IIIT Hyderabad, Telangana, India\\
\email{vidhi.rathore@research.iiit.ac.in, naresh.manwani@iiit.ac.in}}
%
\maketitle              
\begin{abstract}
Principal Component Analysis (PCA) is a widely used method for dimensionality reduction, but it often overlooks fairness, especially when working with data that includes demographic characteristics. This can lead to biased representations that disproportionately affect certain groups. To address this issue, our approach incorporates Joint Eigenvalue Decomposition (JEVD), a technique that enables the simultaneous diagonalization of multiple matrices, ensuring both fair and efficient representations. We formally show that the optimal solution of JEVD leads to a fair PCA solution. By integrating JEVD with PCA, we strike an optimal balance between preserving data structure and promoting fairness across diverse groups. We demonstrate that our method outperforms existing baseline approaches in fairness and representational quality on various datasets. It retains the core advantages of PCA while ensuring that sensitive demographic attributes do not create disparities in the reduced representation. 

\keywords{Fair PCA \and Joint Eigenvalue Decomposition \and Multi-attribute fairness \and Non-binary demographic fairness \and Dimensionality reduction.}

\end{abstract}

\section{Introduction}
The importance of fairness in machine learning has garnered significant attention in recent years, particularly as models are increasingly employed in high-stakes domains such as criminal justice, finance, and healthcare. For instance, the COMPAS algorithm \cite{angwin2016compas}, widely used in the U.S. for criminal risk assessment, was found to disproportionately label Black defendants as high-risk for reoffending, raising concerns about racial bias. Similarly, facial recognition systems \cite{turing_bias_facial_recognition} have been shown to have higher error rates for people of a particular colour, further highlighting the potential for biased outcomes in ML systems. These cases exemplify the critical need for fairness in machine learning, and a vast body of literature has been dedicated to addressing these concerns.

Recent literature emphasizes the need for fairness-aware algorithms that achieve high accuracy and mitigate bias. For instance, the works of \cite{hardt2016}, \cite{zafar2015} and \cite{agarwal2018} have laid foundational metrics for assessing fairness in ML models.

The role of fairness in dimensionality reduction has also been explored in many studies. Dimensionality reduction techniques, such as Principal Component Analysis (PCA), are essential for simplifying complex datasets and improving computational efficiency. However, when applied to data that contain sensitive attributes (e.g., gender, race), these techniques can inadvertently amplify biases by distorting relationships between features in the reduced space. This can lead to unfair treatment of certain groups, as biased representations are carried out in subsequent ML models.
Fair Principal Component Analysis (Fair PCA) has gained significant attention as an extension of classical PCA, aiming to balance the trade-off between dimensionality reduction and fairness considerations in diverse applications. 
Approaches for fair PCA vary based on the fairness criterion used. Equitable reconstruction loss \cite{Samadi2018,Kamani2022}, maximum mean discrepancy (MMD) \cite{Lee2022}, equitable reconstruction error \cite{Pelegrina2023}, representations found by fair PCA leading fair classifiers \cite{Olfat2019}, null-it-out based approach \cite{Lee2024} are some of the recent approaches for fair PCA. In this paper, we use equitable reconstruction losses as fairness criteria. Such a criterion is used in \cite{Samadi2018,Kamani2022} also. However, the algorithm proposed in \cite{Samadi2018} used SDP which has time complexity of $\mathcal{O}(d^{6.5}\log\frac{1}{\epsilon})$, making it less scalable for higher dimensions.
We simplify the expression for reconstruction loss and offer a more scalable algorithm for fair PCA. Our fair PCA approach boils down to the joint eigenvalue decomposition (JEVD) problem \cite{Andre2015},  \cite{Mesloub2013}, \cite{Mesloub2018}, \cite{Remi2020}, \cite{Luciani2010}. Our key contributions are as follows.
\begin{itemize}
    \item We simplify the reconstruction loss expression for PCA which plays a vital role in our fair PCA formulation.
    \item We show that minimizing JEVD objective function leads to fair PCA.
    \item We propose a fast JEVD-based fair PCA algorithm called JEVD-PCA. The proposed approach enjoys the time complexity of $\mathcal{O}(d^3)$. 
    \item We provide a thorough experimental comparison of the proposed approach with different fair PCA baselines on multiple datasets. We evaluate different metrics to show the efficacy of the proposed approach.
\end{itemize}

\section{Related Works}
\subsection{Fair PCA}
In \cite{Samadi2018}, a polynomial-time algorithm for Fair PCA is proposed that minimizes the maximum reconstruction loss between two groups. 
The major computational task in \cite{Samadi2018} involves solving a Semidefinite Program (SDP) which has time complexity \( \mathcal{O}(d^{6.5} \log\frac{1}{\epsilon}) \), where $d$ is the dimension of the data. 
The SDP step dominates the overall runtime, making the algorithm computationally expensive for moderate values of \( d \), limiting its scalability for large problems. 
Another notable study by \cite{Olfat2019} proposes a quantitative definition of fair PCA, which ensures that sensitive attributes cannot be inferred from the dimensionality-reduced data. 
This approach also solves semidefinite programs (SDPs) to achieve fair PCA.
In \cite{Lee2022}, the authors introduce a new mathematical definition of fairness for PCA, utilizing the Maximum Mean Discrepancy (MMD) as a metric. 
Their approach formulates the fair PCA problem as a constrained optimization problem over the Stiefel manifold, allowing for efficient computation and providing local optimality guarantees. 
In \cite{Lee2024}, the authors introduce the "Null It Out" approach, which preserves variance while nullifying sensitive attribute information. They establish a new statistical framework, probably approximately fair and optimal (PAFO)-learnability, providing guarantees about the learnability of Fair PCA. The paper proposes a memory-efficient algorithm that enables data processing in a streaming fashion. 
The work \cite{Pelegrina2023} introduces a new PCA algorithm that explicitly incorporates fairness in the dimensionality reduction process. 
This algorithm formulates an optimization problem combining both the overall reconstruction error and a fairness measure. 
\cite{Kamani2022} introduces the concept of Pareto fair PCA, allowing optimal trade-offs between learning objectives and fairness criteria. The paper provides 
theoretical guarantees on convergence to optimal compromises or Pareto stationary points.

\subsection{Joint Eigenvalue Decomposition (JEVD)}
Joint Eigenvalue Decomposition (JEVD) has emerged as a foundational technique in multivariate signal processing and machine learning, enabling the simultaneous diagonalization of multiple matrices to uncover latent structures and shared properties across datasets. 
An efficient and stable JEVD method is introduced in \cite{Mesloub2018}, which leverages Generalized Givens Rotations. Their formulation for JEVD incorporates Shear rotations to minimize the departure from symmetry in matrices, which results in improved performance and stability. 
The algorithm proposed in \cite{Mesloub2013} extends the capabilities of JEVD to non-orthogonal joint diagonalization (NOJD) of complex matrices. This algorithm, 
integrates Givens and Shear rotations to enhance convergence rates and estimation accuracy while simultaneously reducing computational overhead. 
The JDTM algorithm introduced in \cite{Luciani2010} targets the JEVD of real non-defective matrices. Using a Jacobi-like procedure based on polar matrix decomposition, JDTM introduces a new optimization criterion specifically tailored for hyperbolic matrices. 
The algorithm described in \cite{Andre2015} employs a novel iterative method that bypasses traditional sweeping procedures. Instead, it updates the entire matrix of eigenvectors in each iteration using a first-order approximation of the inverse eigenvector matrix. The algorithm achieves competitive performance with significantly reduced numerical complexity by striking a balance between iteration count and computational cost. Expanding on this, \cite{Remi2020} introduces a JEVD method based on Taylor Expansion (JDTE), which employs a first-order Taylor expansion for multiplicative updates. This approach enables simultaneous optimization of all parameters at each iteration, a strategy previously unexplored in JEVD algorithms. 


\section{Preliminaries}

Let $X\in \mathbb{R}^{n\times d}$ be the data matrix with \( n \) samples, each of dimension \( d \). Let \( \Vert \cdot \Vert_F \) denotes the Frobenius norm \cite{doi:10.1137/1.9781421407944} and $I$ denotes identity matrix . Let $s\in\{A,B\}$ be the sensitive attribute and $X_A,\;X_B$ be the submatrices of $X$ corresponding to subgroups $A$ and $B$ respectively. Also, let $n_A$ and $n_B$ be the number of rows in $X_A$ and $X_B$ respectively. 

\begin{definition}
    {\bf Reconstruction Error of an Orthonormal Projection Matrix $U$: }Reconstruction error of an orthonormal projection matrix $U\in \mathbb{R}^{d\times r}$ ($U^\top U=I)$ on matrix $X$ characterizes how well it reconstructs matrix $X$. It is defined as the Frobenius norm of the difference between matrix $X$ and the projected matrix $XUU^\top$.  
    \begin{align}
        R(U;X)=\frac{1}{n}\Vert X-XUU^\top\Vert^2_F
        \label{eq:recons-error}
    \end{align}
\end{definition}
Note that $\frac{1}{n}\Vert X-XUU^\top\Vert^2=
\frac{1}{n}\text{Trace}(XX^\top)-\frac{1}{n}\text{Trace}(U^\top X^\top XU)$. 

\begin{definition}
    {\bf Reconstruction Loss \cite{Samadi2018}:} Let $X\in \mathbb{R}^{n\times d}$ and $U^*$ be the optimal orthogonal projection matrices of rank $r$ that minimizes reconstruction error. Thus, $U^*=\arg\min_{U\in \mathbb{R}^{d\times r};\;U^\top U=I} R(U;X)$. Then, for any matrix $U\in \mathbb{R}^{d\times r}$, the reconstruction loss $L(U;X)$ is the difference between the reconstruction errors corresponding to $U$ and $U^*$. Thus, $L(U;X)=R(U;X)-R(U^*,X)$. 
\end{definition}
Thus, reconstruction loss captures how poorly projection matrix $U$ performs compared to the optimal $r$-rank projection matrix $U^*$ in terms of reconstruction error.

\subsection{Principal Component Analysis (PCA)}
The goal of standard PCA is to find an orthogonal projection matrix \( U \in \mathbb{R}^{d \times r} \) that minimizes the reconstruction error of the data. Thus, standard PCA solves the following optimization problem to find the projection matrix $U^*$.
\begin{align*}
U^*={\arg\min}_{U\in\mathbb{R}^{d\times k};\;U^\top U = I} \;R(U;X) ={\arg\max}_{U\in\mathbb{R}^{d\times k};\;U^\top U = I} \;\text{Trace}(U^\top X^\top XU)
\end{align*}
Note that, if the data matrix $X$ is mean centered, then $\frac{1}{n}X^\top X=\Sigma$, where $\Sigma$ is covariance matrix. For a mean centered matrix $X$, $\frac{1}{n}\text{Trace}(U^\top X^\top XU)=\text{Trace}(U^\top \Sigma U)=\sum_{i=1}^r\mathbf{u}_i^\top \Sigma \mathbf{u}_i$, where $\mathbf{u}_i$ is the $i^{th}$ column of matrix $U$. Thus, minimizing reconstruction error is equivalent to maximizing the total variance of the projected matrix. To achieve this, PCA finds matrix $U$ whose columns are top $k$ eigenvectors of covariance matrix $\Sigma$ \cite{Jolliffe2011}. Standard PCA does not account for possible disparities between different groups within the dataset.

\subsection{Fair PCA}
PCA finds the optimal projection matrix $U^*$, which minimizes the overall reconstruction error. Reconstruction errors $R(U^*;X_A)$ and $R(U^*;B)$ corresponding to sensitive groups $A$ and $B$ might be far apart. The reason is that the subspaces associated with the two groups might be totally different. We can find individual optimal projection matrices $U_A^*$ and $U_B^*$ for groups $A$ and $B$ separately, but it poses several legal and ethical concerns \cite{Lipton2017DoesMM}. Thus, the objective of the fair PCA is to find a single projection matrix $U_{fair}$ using the entire data, which achieves similar reconstruction loss for both the groups $A$ and $B$. The formal definition of this Fair PCA criteria is given below.
\begin{definition}
    {\bf Fair PCA: }Let $X\in \mathbb{R}^{n\times d}$ be the data matrix and let $S$ be the sensitive attribute which takes two values $S\in\{A,B\}$. Thus, Fair PCA for data $X$ is the one which finds orthonormal projection matrix $U\in \mathbb{R}^{d\times r}$ achieving equal reconstruction loss for group $A$ and group $B$.
    In other words, projection matrix leads $U_{\text{Fair}}\in \mathbb{R}^{d\times r}$ leads to Fair-PCA if 
\begin{align}
\label{eq:fair-pca-condition}
L(X_A;U_{\text{Fair}})=L(X_B;U_{\text{Fair}})
\end{align}
\end{definition}

\paragraph{\bf Objective Function for Fair PCA: }We need to find projection matrix $U\in \mathbb{R}^{d\times r}$ ($U^\top U=I$) for which $L(X_A;U)=L(X_B;U)$. Such a projection matrix $U$ can be achieved by minimizing the maximum of $L(X_A;U)$ and $L(X_B;U)$ \cite{Samadi2018}. 
\begin{align}
\label{eq:obj}
  U_{\text{Fair}}=  \underset{U\in \mathbb{R}^{d\times r};\;U^\top U=I}{\arg\min}\;\;
    \max\left(L(X_A;U),L(X_B;U)\right)
\end{align}

\begin{theorem}\cite{Samadi2018}
    Let $U_{\text{Fair}}$ be the solution to the optimization problem (\ref{eq:obj}), then $L(X_A;U_{\text{Fair}})=L(X_B;U_{\text{Fair}})$.
\end{theorem}
 
\subsection{Joint EigenValue Decomposition (JEVD)}

The JEVD method aims to simultaneously diagonalize a set of \( K \) non-defective matrices \( M_k,\;k=1\ldots K \). Thus, JEVD finds a matrix $U$ such that
\[
M_k = U D_k U^{-1}, \quad \forall k = 1, \dots, K,
\]
where \( D_k \) are diagonal matrices. Thus, columns of \( U \) are eigenvectors of all $M_k,\;k=1\ldots K$. The solution is unique up to a permutation and scaling indeterminacy of \( U \). The uniqueness condition is fulfilled when the rows of the matrix containing the diagonals of \( D_k \) are distinct. Most of the JEVD algorithms find such a matrix $U$ by ensuring $U^{-1}M_k U$ is as much diagonal as possible for each $k\in [K]$. This objective is achieved by minimizing the following.
\begin{align}
\label{eq:jevd-obj}
C(U) = \sum_{k=1}^K \Vert \text{ZDiag}\left( U^{-1} M_k U \right) \Vert_F^2,
\end{align}
Here, $\text{ZDiag}:\mathbb{R}^{c\times c}\rightarrow \mathbb{R}^{c\times c-1}$ outputs a off-diagonal elements of input matrix. Thus, the objective function makes each matrix $U^{-1}M_kU$ diagonal by forcing off-diagonal elements to be zero.

\section{Proposed Fair PCA Approach}
Here, we propose a fair PCA approach for two sensitive groups. Its extension for multiple sensitive groups is trivial. The proposed Fair-PCA approach has a dual objective: (a) learn lower dimensional representation by learning a projection matrix, and (b) representations learnt should be fair across different groups regarding reconstruction loss. 
To learn fair PCA, we find the projection matrix $U\in \mathbb{R}^{d\times r}$ which minimizes reconstruction losses $L(X_A;U)$ and $L(X_B;U)$ simultaneously. 
Going ahead, we first simplify the expression for reconstruction loss $L(X;U)$. 
\begin{lemma}
\label{lemma1}
    Given matrix $X\in \mathbb{R}^{n\times d}$, its reconstruction loss of an orthonormal projection matrix $U\in \mathbb{R}^{d\times r}$ can be expressed as follows.
    \begin{align*}
        L(X;U)
        =\text{Trace}\left(U^\top \frac{1}{n}\left[\frac{1}{r}\sum_{k=1}^r\sigma_k^2(X)I_{d\times d}-X^\top X\right]U\right).
    \end{align*}
\end{lemma}
\begin{proof}
    Using the definition of reconstruction loss, we get the following.
    \begin{align}
        \nonumber & L(X;U)= R(X;U)-R(X;U^*)= \frac{1}{n}\text{Trace}(U^{*\top} X^\top XU^*)-\frac{1}{n}\text{Trace}(U^\top X^\top XU)\\
         \nonumber &= \frac{1}{r}\;\text{Trace}(U^{*\top} \frac{X^\top X}{n}U^*)\text{Trace}(U^\top U)-\text{Trace}(U^\top \frac{X^\top X}{n}U)\\
        &=\text{Trace}\left(U^\top \left[\frac{1}{r}\;\text{Trace}(U^{*\top}\frac{X^\top X}{n}U^*)I_{d\times d}-\frac{1}{n}X^\top X\right]U\right)
        \label{loss-new}
    \end{align}
    In the third step, we have used the fact that $\text{Trace}(U^\top U)=r$. In the fourth step, we have used the property of trace, that is $\text{Trace}(A+B)=\text{Trace}(A)+\text{Trace}(B),\;\forall A,B\in \mathbb{R}^{m\times m}$.
    Using Theorem 23.2 in \cite{10.5555/2621980}, we know that for $U^*$, $\text{Trace}(U^{*\top} X^\top XU^*)=\sum_{k=1}^r \sigma_k^2(X)$,
    where $\sigma_k(X)$ is the $k^{th}$ largest singular value of $X$. Using this, in Eq.(\ref{loss-new}), we get
    \begin{align*}
        L(X;U)
        =\text{Trace}\left(U^\top \frac{1}{n}\left[\frac{1}{r}\sum_{k=1}^r\sigma_k^2(X)I_{d\times d}-X^\top X\right]U\right)\quad\quad\quad \blacksquare
    \end{align*} 
\end{proof}

\subsection{Objective for Fair PCA}
    Let $U_A$ and $U_B$ be the optimal $r$-rank orthonormal projection matrices minimizing reconstruction errors $R(X_A;U)$ and $R(X_B;U)$ respectively. Let $M_{X_A}$ and $M_{X_B}$ be the following matrices.
\begin{align*}
    M_{X_A}&=\frac{1}{n_A}\left[X_A^\top X_A-\frac{1}{r}\sum_{k=1}^r \sigma_k^2(X_A)I_{d\times d}\right]\\
    M_{X_B}&=\frac{1}{n_B}\left[X_{B}^\top X_B-\frac{1}{r}\sum_{k=1}^r \sigma_k^2(X_B)I_{d\times d}\right]
\end{align*}
Where $\sigma_k(X_A)$ and $\sigma_k(X_B)$ be the $k^{th}$ largest singular values of $X_A$ and $X_B$ respectively. 

A fair projection matrix $U$ is the one that jointly minimizes $L(X_A;U)$ and $L(X_B;U)$. Then, with the help of Lemma~\ref{lemma1}, this is equivalent to $U$ that jointly maximizes $\text{Trace}(U^\top M_{X_A}U)$ and $\text{Trace}(U^\top M_{X_B}U)$. Which is the same as the joint diagonalization of matrices $M_{X_A}$ and $M_{X_B}$. Thus, we propose a fair PCA problem as a joint eigenvalue decomposition of matrices $M_{X_A}$ and $M_{X_B}$. To find $U_{\text{Fair}}$, we minimize the objective used in joint eigenvalue decomposition as follows.
\begin{align}
    \nonumber U_{\text{Fair}}&=\underset{U\in \mathbb{R}^{d\times d};\;U^\top U=I}{\arg\min}\;\;
    C(U)\\
    \label{eq:fair-jevd}&=\underset{U\in \mathbb{R}^{d\times d};\;U^\top U=I}{\arg\min}\;\;\Vert \text{ZDiag}({U}^{-1}M_{X_A}U)\Vert^2 + \Vert \text{ZDiag}({U}^{-1}M_{X_B}U)\Vert^2
\end{align}
We now show that if $U^*$ is the global minimizer of $C(U)$ in eq.(\ref{eq:fair-jevd}) such that $C(U^*)=0$, then $U^*$ satisfies the fair PCA condition described in eq.(\ref{eq:fair-pca-condition}).
\begin{theorem}
    Let $U^*=\arg\min_{U\in \mathbb{R}^{d\times d}}\;C(U)$, where $C(U)$ is described in eq.(\ref{eq:jevd-obj}) and $C(U^*)=0$, then $L(X_A;U^*)=L(X_B;U^*)$.
\end{theorem}
\begin{proof}
    $U^*=\arg\min_{U\in \mathbb{R}^{d\times d}}\;C(U)$ and $C(U^*)=0$. Then $\Vert \text{ZDiag}({U^*}^{-1}M_{X_A}U^*\Vert^2 + \Vert \text{ZDiag}({U^*}^{-1}M_{X_B}U^*\Vert^2=0$. Using the property of the norm, we see that elements of $\text{ZDiag}({U^*}^{-1}M_{X_A}U^*)$ and $\text{ZDiag}({U^*}^{-1}M_{X_B}U^*)$ are all zero. Thus, ${U^*}^{-1}M_{X_A}U^*=D_A$ and ${U^*}^{-1}M_{X_B}U^*=D_B$, where $D_A$ and $D_B$ are some diagonal matrices. Since $M_{X_A}$ and $M_{X_B}$ are symmetric matrices, $U^*$ must be orthonromal. Thus, ${U^*}^{-1}=U^{*\top}$. Thus, $U^*=\arg\max_{U\in \mathbb{R}^{d\times d}}\;\text{Trace}(U^\top M_{X_A}U)$ and $U^*=\arg\max_{U\in \mathbb{R}^{d\times d}}\;\text{Trace}(U^\top M_{X_B}U)$ \cite{ghojogh2023}. Thus, 
    \begin{align*}
       U^*&= \underset{U\in \mathbb{R}^{d\times d};\;U^\top U=I}{\arg\max}\;\;
    \min\left(U^\top X_A U,\;U^\top X_B U\right)\\
    &=\underset{U\in \mathbb{R}^{d\times d};\;U^\top U=I}{\arg\max}\;\;
    \min\left(-L(X_A; U),\;-L(X_B;U)\right)\\
    &=\underset{U\in \mathbb{R}^{d\times d};\;U^\top U=I}{\arg\min}\;\;
    \max\left(L(X_A; U),\;L(X_B;U)\right)
    \end{align*}
    Using Theorem 1, we see that for such a $U^*$, we have $L(X_A;U^*)=L(X_B;U^*). \blacksquare$
\end{proof}
Thus, we have shown that the solution of the JEVD optimization problem satisfies the fair PCA condition.

\subsection{\bf JEVD Algorithm for Fair PCA}
To find matrix \( U \) that minimizes the objective function $C(U)$ in eq.(\ref{eq:fair-jevd}), the following iterative approach is used. We use the JEVD algorithm proposed in \cite{Andre2015}. The JEVD algorithm begins by initializing \( T_A^{(1)} \) as \( M_{X_A} \) and \( T_B^{(1)} \) as \( M_{X_B} \). At each iteration \( i \), the matrix \( B_i \) is computed to minimize the diagonalization criterion: $C(B_i) = \Vert \text{ZDiag}\left( B_i^{-1} T_A^{(i)} B_i \right) \Vert_F^2+\Vert\text{ZDiag}\left(B_i^{-1}T_B^{(i)}B_i\right)\Vert_F^2$.
\begin{algorithm}[t]
\caption{JEVD Based Fair PCA Algorithm (JEVD-PCA)}
\label{algo1}
\begin{algorithmic}[1] 
\State \textbf{Input:} Matrices \( M_{X_A},\;M_{X_B}\), maximum iterations \( S_{\text{max}} \), stopping criterion \( \epsilon \)
\State \textbf{Initialize:} \( U = I \), \( i = 1 \). Set \( T_A^{(1)} = M_{X_A} \) and \( T_B^{(1)} = M_{X_B} \).
\While{\( C(U)>\epsilon \) is not satisfied and \( i \leq S_{\text{max}} \)}
    \For{\( m = 1 \) to \( d \)}
        \For{\( n = 1 \) to \( d \)}
            \If{\( m \neq n \)}
                \State Compute \( V_{mn} \) using the update rule:
                \State \(V_{mn} = - \frac{ O_A(mn) \left( \Lambda_A(mm) - \Lambda_A(nn) \right)+O_B(mn) \left( \Lambda_B(mm) - \Lambda_B(nn)\right)}{\left( \Lambda_A(mm) - \Lambda_A(nn) \right)^2+\left( \Lambda_B(mm) - \Lambda_B(nn) \right)^2}\)
            \EndIf
        \EndFor
    \EndFor
    \State Compute \( B_i = I + V \)
    \State Compute \( B_i^{-1} \)
    \For{\(s\in\{A,B\}\)}
        \State Update \( T_s^{(i+1)} = B_i^{-1} T_s^{(i)} B_i \)
    \EndFor
    \State Update \( U \gets U B_i \)
    \State Increment iteration: \( i \gets i + 1 \)
\EndWhile
\State \textbf{Output:} Matrix \( U \)
\end{algorithmic}
\end{algorithm}
To simplify the process, the matrix \( B_i \) is approximated as \( B_i = I + V_i \), where \( V_i \) represents the correction matrix at iteration \( i \). Let $T_s^{(i)}=\Lambda_s^{(i)}+O_s^{(i)},\;s\in\{A,B\}$, where $\Lambda_s^{(i)}$ is the diagonal matrix containing diagonal entries of $T_s^{(i)}$ and $O_s^{(i)}=\text{ZDiag}(T_s^{(i)})$. The optimization process is further simplified by applying a first-order Taylor expansion to the criterion, yielding the approximate form \cite{Andre2015} as $
C_a(V) \approx  \Vert \text{ZDiag}\left( O_A - V\Lambda_A + \Lambda_A V \right) \Vert_F^2 +\Vert \text{ZDiag}\left( O_B - V \Lambda_B + \Lambda_B V\right) \Vert_F^2$.
The correction matrix \( V \) is updated to minimize this quadratic form. The update rule for each element \( V_{mn} \) is:
\[
V_{mn} = - \frac{ O_A(mn) \left( \Lambda_A(mm) - \Lambda_A(nn) \right)+O_B(mn) \left( \Lambda_B(mm) - \Lambda_B(nn)\right)}{\left( \Lambda_A(mm) - \Lambda_A(nn) \right)^2+\left( \Lambda_B(mm) - \Lambda_B(nn) \right)^2}.
\]
The process iterates until the stopping criterion is satisfied. Complete details of the JEVD-based Fair PCA algorithm discussed here are given in Algorithm~\ref{algo1}.
The complexity of the JEVD-PCA algorithm is $\mathcal{O}(d^3)$ \cite{Andre2015} where $d$ is the data dimension.
In contrast, the method in \cite{Samadi2018} which solves SDP has complexity \( \mathcal{O}(d^{6.5}\log\frac{1}{\epsilon}) \).

\section{Experiments}

This section presents the experimental setup and results of the proposed method. For benchmarking, we compare our approach against three models: the standard Principal Component Analysis (PCA) and two additional methods from the literature, \cite{Samadi2018} and \cite{Pelegrina2023} referred to as \texttt{FairPCA\_SDP\_LP} and \texttt{covFairPCA} respectively.
\subsection{Datasets}  
We evaluated our method on the following four datasets: (a) Diabetes Health Indicators Dataset \cite{teboul_2023},
(b) LSAC Dataset \cite{pelegrina_2023}, 
(c) National Poll on Healthy Aging (NPHA) Dataset \cite{national_poll_on_healthy_aging_(npha)_936} and 
(d) Estimation of Obesity Levels Based on Eating Habits and Physical Condition \cite{estimation_of_obesity_levels_based_on_eating_habits_and_physical_condition__544}. 
For all the datasets, we considered gender as the sensitive attribute. 

\subsection{Baselines for Comparison}
We compare our approach against the following baseline methods. (a) {\bf Standard PCA: }We compare with standard PCA which actually faces bias issues. (b) \textbf{FairPCA\_SDP\_LP (\cite{Samadi2018}):} This approach is based on minimizing the objective function (\ref{eq:obj}). It solves an SDP problem to find a projection matrix that balances reconstruction losses across groups. 
(c) \textbf{covFairPCA (\cite{Pelegrina2023}):} The paper proposes two algorithms for fair dimensionality reduction: constrained Fair PCA (c-FPCA). c-FPCA ensures fairness by minimizing disparities in reconstruction errors between sensitive groups while preserving the quality of representation.

Baselines \cite{Samadi2018} and \cite{Pelegrina2023} were chosen as they represent established methods in Fair PCA, providing a solid basis for comparing and validating our approach.

\subsection{Evaluation Metrics}
We use the following metrics to evaluate the performance of the proposed approach and the baseline approaches. (a) \textbf{Reconstruction Error:} As explained in eq.(\ref{eq:recons-error}), this metric measures the squared distance between the original matrix $X$ and the reconstructed matrix $XUU^\top$. (b) \textbf{Variance Explained:} This metric quantifies the proportion of the total data variance captured by the reduced representation. It is calculated as: $\text{Variance Explained} = \frac{\text{Var}(XU)}{\text{Var}(X)}$, where \(\text{Var}(\cdot)\) represents the sum of variances across all features in the data. (c) \textbf{Maximum Mean Discrepancy (MMD\(^2\)):} This metric computes the statistical distance between distributions of two groups in a dataset. It is given by $\text{MMD}^2 = \Vert\mu_A - \mu_B\Vert^2$, where \(\mu_A\) and \(\mu_B\) are the means of the two groups in the dataset being compared.  
MMD$^2$ is used to capture the fairness of the PCA algorithms \cite{Lee2022}.
\begin{figure}[t!]
\begin{center}
\begin{tabular}{c|ccc}
 & Reconstruction Error & Variance Explained & MMD$^2$ \\
\hline
  \begin{turn}{90}{\footnotesize{Diabetes Dataset}}\end{turn} & 
  \includegraphics[scale=0.18]{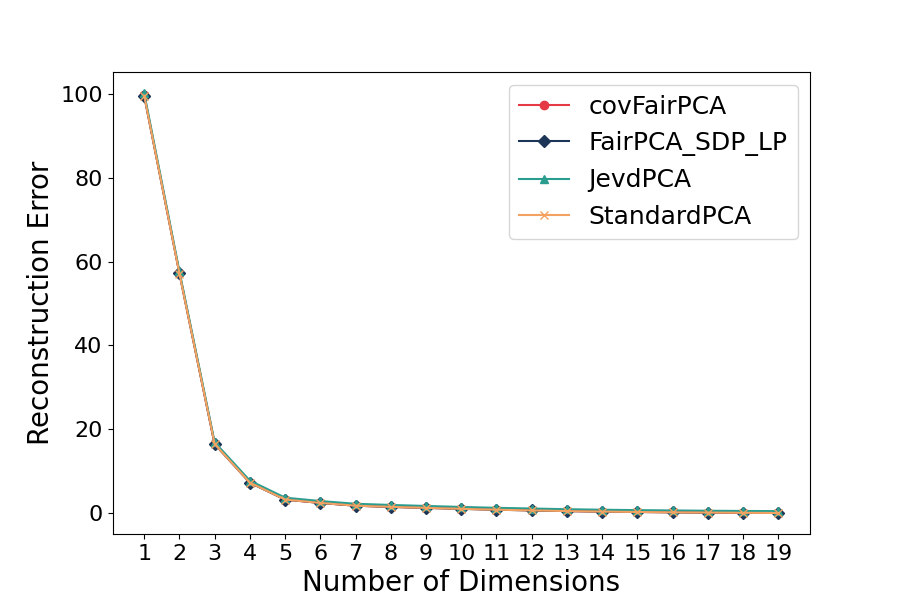} &
   %
   \includegraphics[scale=0.18]{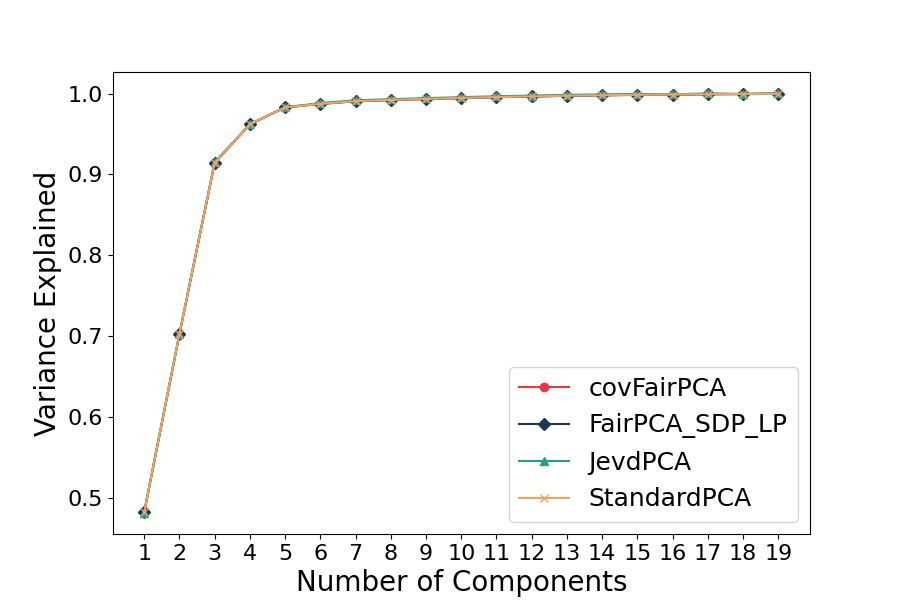}&
   %
   \includegraphics[scale=0.18]{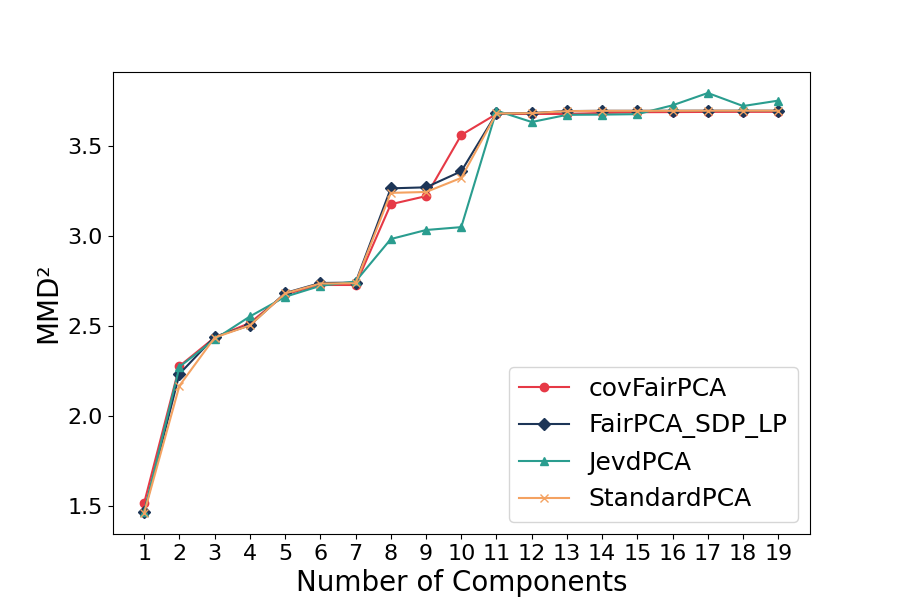}\\
   \hline
   %
   %
    \begin{turn}{90}{\footnotesize{LSAC Dataset}}\end{turn} &
    %
    \includegraphics[scale=0.18]{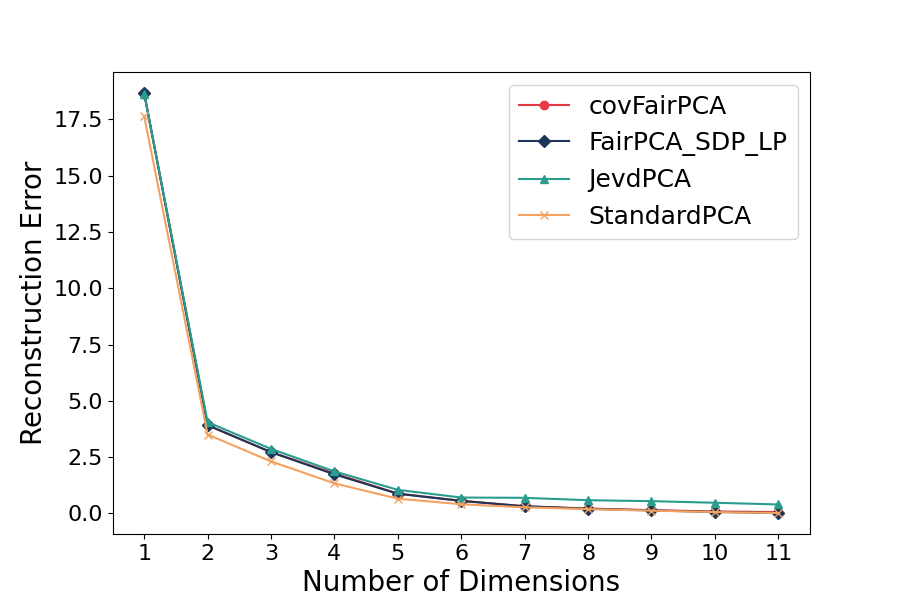}
    &    \includegraphics[scale=0.18]{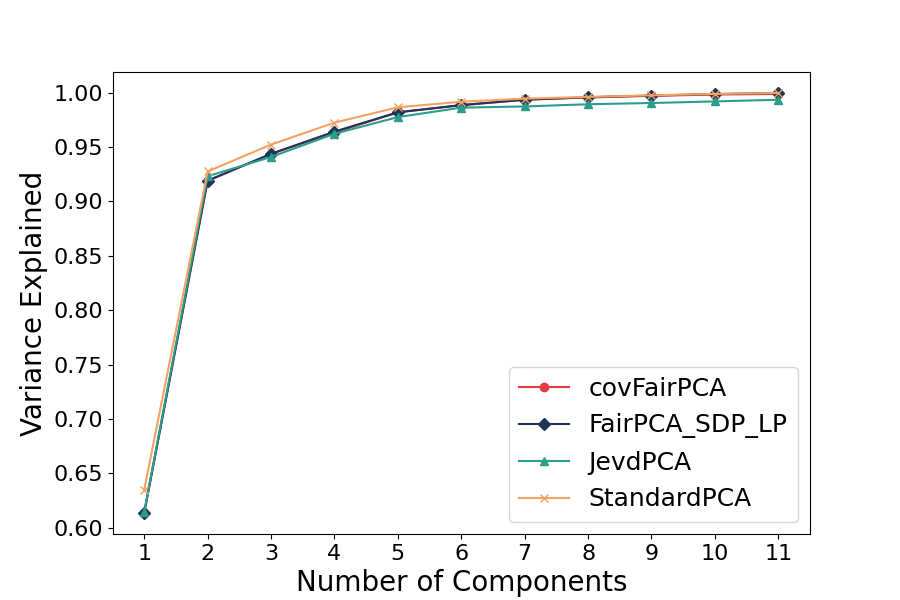}&
     %
     \includegraphics[scale=0.18]{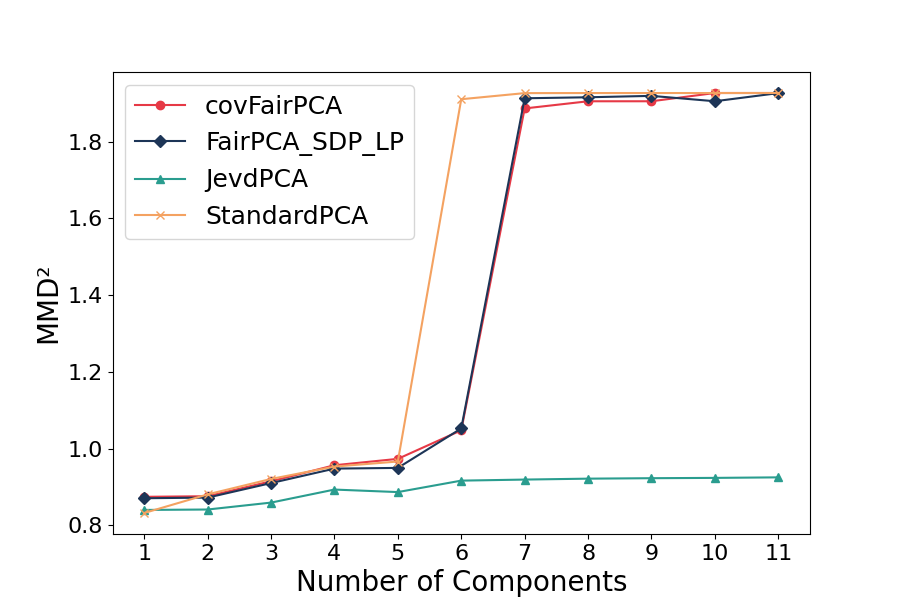}\\
     \hline
    %
    %
     \begin{turn}{90}{\footnotesize{NPHA Dataset}}\end{turn}  & \includegraphics[scale=0.18]{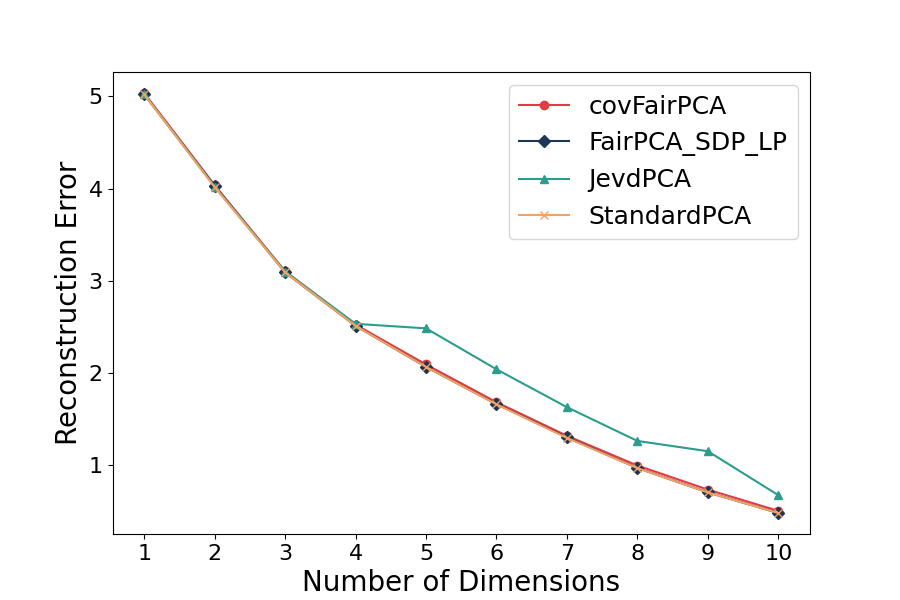}&
    \includegraphics[scale=0.18]{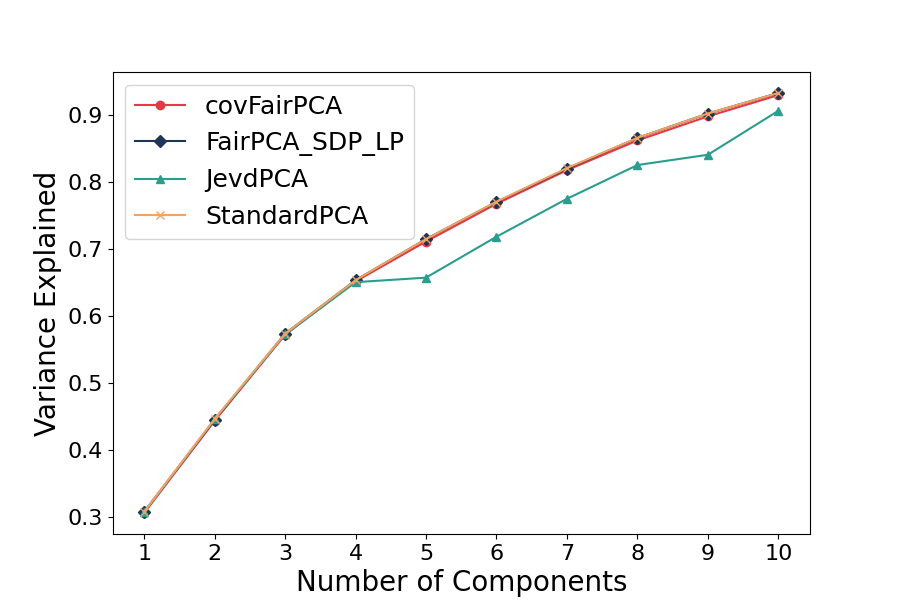}&
      %
      \includegraphics[scale=0.18]{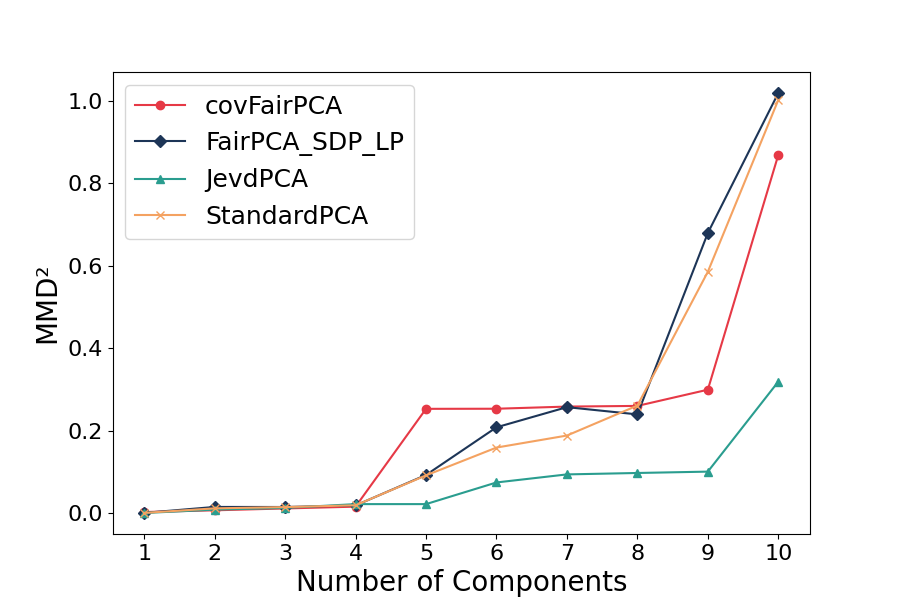}\\
      \hline
    \begin{turn}{90}{\footnotesize{Obesity Dataset}}\end{turn}&
    %
     %
     \includegraphics[scale=0.18]{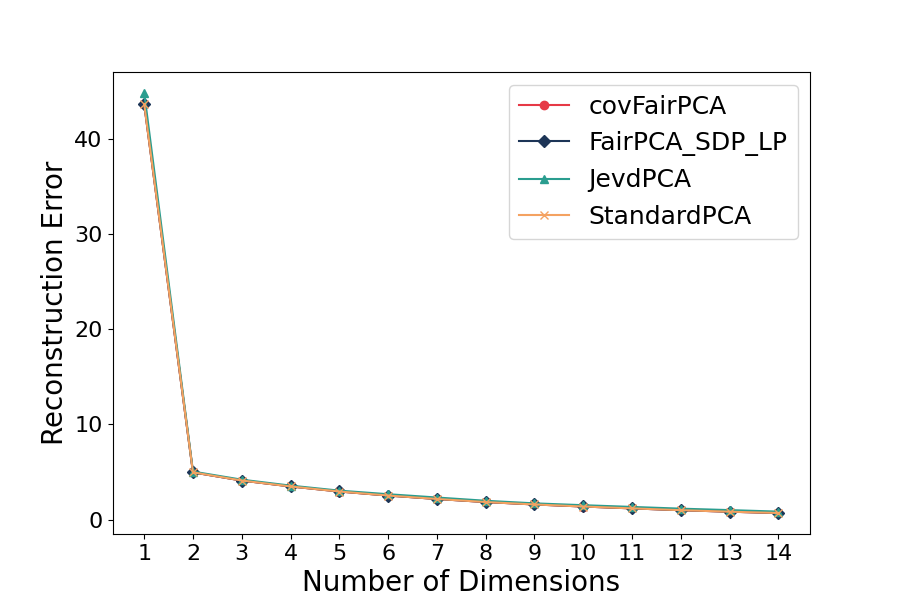}&
    %
    \includegraphics[scale=0.18]{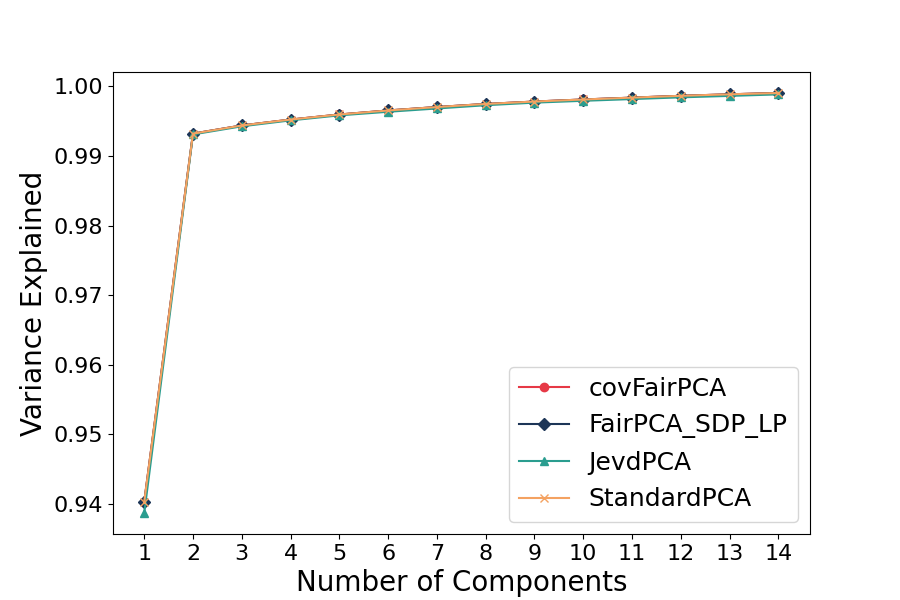}&
     %
     \includegraphics[scale=0.18]{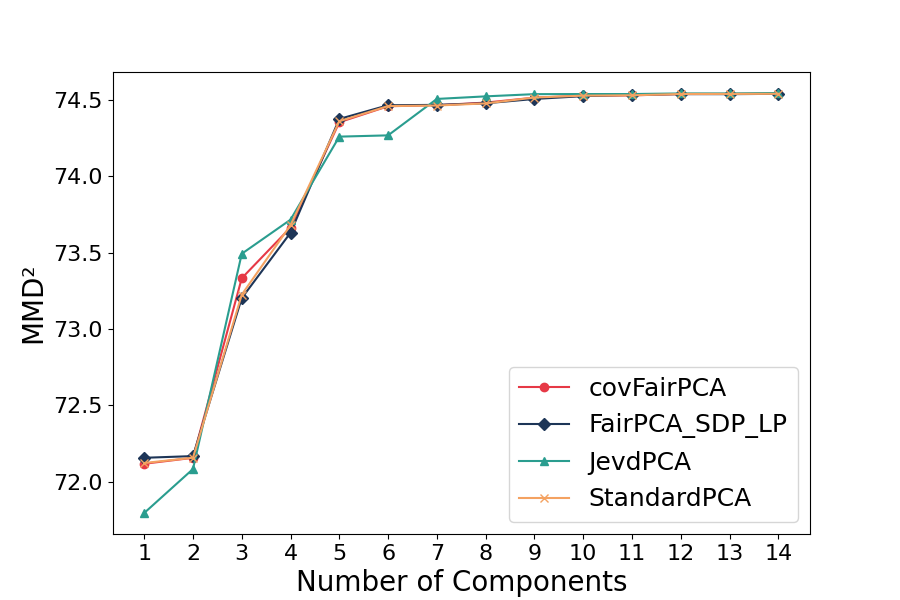}\\
     \hline
    %
    %
      %
      %
    %
    \end{tabular}
        \caption{Comparison Results on Reconstruction Error, Variance Explained, and MMD\(^2\) for the proposed approach JEVDPCA and baseline approaches.}
    \label{fig:metrics_visualization}
    \end{center}
\end{figure}

\subsection{Experimental Results: }
Figure~\ref{fig:metrics_visualization} shows comparison results of the proposed approach JEVDPCA with baseline approaches on three metrics (reconstruction error, variance explained and MMD$^2$). The first column shows results on reconstruction error (lower is better), which quantifies the fidelity of the reduced representation to the original data. The middle plots illustrate Variance Explained (higher is better), reflecting how much of the original data's variability is captured by the representation. Finally, the right plots display MMD\(^2\) (lower is better), measuring the fairness of the representation by assessing the distributional similarity between sensitive groups.

We observe that JEVD-PCA performs better than all the baseline approaches on MMD$^2$ for LSAC and NPHA datasets, which has a huge gap. On Diabetes and Obesity datasets, JEVD-PCA performs comparable to other fair PCA baselines. On the other hand, in terms of reconstruction error and variance explained, the proposed approach JEVD-PCA performs comparably to the other baselines on all the datasets except for NPHA dataset. For NPHA dataset, JEVD-PCA underperforms by a small margin for higher values of $r$. Thus, we see that JEVD-PCA achieved better fairness results compared to the other baseline while maintaining similar reconstruction errors.

Table~\ref{tab:execution_times} shows the training time comparison between the proposed approach and baseline approaches. FairPCA\_SDP\_LP \cite{Samadi2018} also uses equitable reconstruction loss as a criterion for fair PCA. We observe that the proposed approach is much faster than the FairPCA\_SDP\_LP in training time. JEVD-PCA does better on Diabetes and LSAC datasets compared to covFairPCA \cite{pelegrina_2023}. Compared to standard PCA, JEVD-PCA  takes more time as it is solving trying to diagonalize two matrices simultaneously. 


\begin{table}[]
\centering
\caption{Training times (in seconds) for JEVD-PCA and Baseline Approaches.}
\label{tab:execution_times}
\begin{tabular}{|c|c|c|c|c|}
\hline
\textbf{Method} & \textbf{Diabetes} & \textbf{LSAC} & \textbf{NPHA} & \textbf{Obesity} \\ \hline
covFairPCA     & 55.9304  & 0.8066  & 0.3013 & 1.1090   \\ \hline
FairPCA\_SDP\_LP & 136.7052 & 13.6891 & 0.7191 & 514.3610 \\ \hline
JevdPCA        & 5.1725   & 0.4644  & 0.7174 & 12.5314  \\ \hline
StandardPCA    & 17.0916  & 0.3382  & 0.0267 & 0.3314   \\ \hline
\end{tabular}
\end{table}

\section{Conclusion and Future Work}
In this paper, we have introduced a novel approach for Fair Principal Component Analysis (Fair PCA) leveraging Joint Eigenvalue Decomposition (JEVD) to ensure fairness during dimensionality reduction. By integrating fairness constraints directly into the JEVD framework, we have demonstrated how our method mitigates bias amplification in the reduced representations while preserving essential data structure. We compared our approach with existing Fair PCA methods, such as \cite{Samadi2018} and \cite{Pelegrina2023}, and showed that our method provides a more balanced performance in terms of both fairness and reconstruction error across sensitive groups. Our experimental results highlight the effectiveness of Fair PCA using JEVD in maintaining equity without sacrificing accuracy, offering a promising avenue for fair machine learning in high-dimensional data.

Looking ahead, there are several areas for improvement. First, we plan to explore how our method scales with larger datasets and how it could be adapted to other dimensionality reduction techniques like Kernel PCA. Additionally, experimenting with other fairness metrics could help refine the trade-offs between fairness and performance. Finally, testing our approach in real-world applications across different domains will be important for assessing its broader impact.

%
%
%
\bibliographystyle{splncs04}

\bibliography{references}

\begin{thebibliography}{10}
\providecommand{\url}[1]{\texttt{#1}}
\providecommand{\urlprefix}{URL }
\providecommand{\doi}[1]{https://doi.org/#1}

\bibitem{national_poll_on_healthy_aging_(npha)_936}
{National Poll on Healthy Aging (NPHA)}. UCI Machine Learning Repository
  (2017), {DOI}: https://doi.org/10.3886/ICPSR37305.v1

\bibitem{estimation_of_obesity_levels_based_on_eating_habits_and_physical_condition__544}
{Estimation of Obesity Levels Based On Eating Habits and Physical Condition }.
  UCI Machine Learning Repository (2019)

\bibitem{agarwal2018}
Agarwal, A., Beygelzimer, A., Dudík, M., Langford, J., Wallach, H.: A
  reductions approach to fair classification. In: ICML (2018)

\bibitem{Remi2020}
André, R., Luciani, X., Moreau, E.: Joint eigenvalue decomposition algorithms
  based on first-order taylor expansion. IEEE Transactions on Signal Processing
   \textbf{PP}, ~1--1 (02 2020)

\bibitem{Andre2015}
André, R., Trainini, T., Luciani, X., Moreau, E.: A fast algorithm for joint
  eigenvalue decomposition of real matrices. pp. 1316--1320 (08 2015)

\bibitem{angwin2016compas}
Angwin, J., Larson, J., Mattu, S., Kirchner, L.: Machine bias.
  \url{https://www.propublica.org/article/machine-bias-risk-assessments-in-criminal-sentencing}
  (2016)

\bibitem{ghojogh2023}
Ghojogh, B., Karray, F., Crowley, M.: Eigenvalue and generalized eigenvalue
  problems: Tutorial (2023), \url{https://arxiv.org/abs/1903.11240}

\bibitem{doi:10.1137/1.9781421407944}
Golub, G.H., Van~Loan, C.F.: Matrix Computations - 4th Edition. Johns Hopkins
  University Press (2013)

\bibitem{hardt2016}
Hardt, M., Price, E., Srebro, N.: Equality of opportunity in supervised
  learning. In: Proceedings of the 30th International Conference on Neural
  Information Processing Systems (2016)

\bibitem{turing_bias_facial_recognition}
Institute, T.T.: Understanding bias in facial recognition technology (2020)

\bibitem{Jolliffe2011}
Jolliffe, I.: Principal Component Analysis, pp. 1094--1096. Springer Berlin
  Heidelberg (2011)

\bibitem{Kamani2022}
Kamani, M.M., Haddadpour, F., Forsati, R., Mahdavi, M.: Efficient fair
  principal component analysis. Machine Learning  \textbf{111},  1--32 (01
  2022)

\bibitem{Lee2024}
Lee, J., Cho, H., Yun, S.Y., Yun, C.: Fair streaming principal component
  analysis: Statistical and algorithmic viewpoint (05 2024)

\bibitem{Lee2022}
Lee, J., Kim, G., Olfat, M., Hasegawa-Johnson, M., Yoo, C.: Fast and efficient
  mmd-based fair pca via optimization over stiefel manifold. vol.~36, pp.
  7363--7371 (06 2022)

\bibitem{Lipton2017DoesMM}
Lipton, Z.C., McAuley, J., Chouldechova, A.: Does mitigating ml's impact
  disparity require treatment disparity? In: Neural Information Processing
  Systems (2017)

\bibitem{Luciani2010}
Luciani, X., Albera, L.: Joint eigenvalue decomposition using polar matrix
  factorization. vol.~6365, pp. 555--562 (09 2010)

\bibitem{Mesloub2018}
Mesloub, A., Belouchrani, A., abed meraim, K.: Efficient and stable joint
  eigenvalue decomposition based on generalized givens rotations. pp.
  1247--1251 (09 2018)

\bibitem{Mesloub2013}
Mesloub, A., abed meraim, K., Belouchrani, A.: A new algorithm for complex
  non-orthogonal joint diagonalization based on shear and givens rotations.
  IEEE Transactions on Signal Processing  \textbf{62} (06 2013)

\bibitem{Olfat2019}
Olfat, M., Aswani, A.: Convex formulations for fair principal component
  analysis. In: Proceedings of the Thirty-Third AAAI Conference on Artificial
  Intelligence. AAAI'19/IAAI'19/EAAI'19 (2019)

\bibitem{pelegrina_2023}
Pelegrina, G.: Datasets for fair pca (fpca) (2023),
  \url{https://github.com/GuilhermePelegrina/Datasets/tree/main/FPCA}

\bibitem{Pelegrina2023}
Pelegrina, G., Duarte, L.: A novel approach for fair principal component
  analysis based on eigendecomposition. IEEE Transactions on Artificial
  Intelligence  \textbf{PP},  1--12 (01 2023)

\bibitem{Samadi2018}
Samadi, S., Tantipongpipat, U.T., Morgenstern, J., Singh, M., , Vempala, S.S.:
  The price of fair pca: One extra dimension. In: Proceedings of Neural
  Information Processing Systems (2018)

\bibitem{10.5555/2621980}
Shalev-Shwartz, S., Ben-David, S.: Understanding Machine Learning: From Theory
  to Algorithms. Cambridge University Press, USA (2014)

\bibitem{teboul_2023}
Teboul, A.: Diabetes health indicators dataset (2023),
  \url{https://www.kaggle.com/datasets/alexteboul/diabetes-health-indicators-dataset},
  accessed: 2024-12-01

\bibitem{zafar2015}
Zafar, M., Valera, I., Rodriguez, M., Gummadi, K.: Fairness constraints: A
  mechanism for fair classification. In: Proceedings of the 21st ACM SIGKDD
  International Conference on Knowledge Discovery and Data Mining (2015)

\end{thebibliography}
\newpage
\appendix
\section{Dataset Details}

\subsection{Diabetes Health Indicators Dataset}
The Diabetes Health Indicators Dataset contains healthcare statistics and lifestyle survey information about individuals, along with their diabetes diagnosis. It consists of 35 features, including demographics, lab test results, and survey responses, with a target variable classifying individuals as diabetic, pre-diabetic, or healthy. The dataset is tabular and multivariate, containing 253,680 instances and 21 features. It was created to understand the relationship between lifestyle and diabetes in the US and was funded by the CDC. Each row represents a study participant, and cross-validation or a fixed train-test split is recommended.

\subsection{Law School Admissions Bar Passage Dataset}
This dataset tracks approximately 27,000 law students from 1991 to 1997 through law school, graduation, and bar exam attempts. Collected for the 'LSAC National Longitudinal Bar Passage Study' by Linda Wightman in 1998, it provides data on the demography, experiences, and outcomes of aspiring lawyers. The dataset is tabular, primarily used for classification, and contains categorical and integer features. It was funded by the Law School Admissions Council (LSAC), with each row representing a law student. Cross-validation or a fixed train-test split is recommended.

\subsection{National Poll on Healthy Aging (NPHA)}

The NPHA dataset focuses on health, healthcare, and policy issues affecting older Americans. This subset is used to develop and validate machine learning algorithms for predicting the number of doctors a respondent sees annually. It is a tabular dataset with 714 instances and 14 categorical features. The dataset was created to provide insights into healthcare concerns of Americans aged 50 and older funded by the University of Michigan. Each row represents a senior survey respondent, and cross-validation or a fixed train-test split is recommended.

\subsection{Estimation of Obesity Levels Based on Eating Habits and Physical Condition}

This dataset includes records from Mexico, Peru, and Colombia, assessing obesity levels based on eating habits and physical condition. It contains 2,111 instances with 16 integer features and supports classification, regression, and clustering tasks. The data was partially synthetically generated using the Weka tool and SMOTE filter, with 23\% collected from users via a web platform. It was created to estimate obesity levels based on lifestyle factors, but funding details are not specified. Each row represents an individual, and cross-validation or a fixed train-test split is recommended.

\section{Tables corresponding to Figure 1}
\begin{table}[htbp]
\centering
\scriptsize
\renewcommand{\arraystretch}{1.4}
\setlength{\tabcolsep}{3pt} 
\caption{Comparison of PCA Methods across Different Dimensions for Diabetes Dataset} 
\begin{tabular}{|c|ccc|ccc|ccc|ccc|}
\hline
\multirow{2}{*}{\textbf{Dims}} & \multicolumn{3}{c|}{\textbf{covFairPCA}} & \multicolumn{3}{c|}{\textbf{FairPCA\_SDP\_LP}} & \multicolumn{3}{c|}{\textbf{JevdPCA}} & \multicolumn{3}{c|}{\textbf{StandardPCA}} \\
\cline{2-13}
 & \textbf{R.E.} & \textbf{V.E.} & $\textbf{MMD}^2$ & \textbf{R.E.} & \textbf{V.E.} & $\textbf{MMD}^2$ & \textbf{R.E.} & \textbf{V.E.} &  $\textbf{MMD}^2$ & \textbf{R.E.} & \textbf{V.E.} &  $\textbf{MMD}^2$ \\
\hline
1  & 99.60 & 0.48 & 1.52 & 99.53 & 0.48 & 1.47 & 100.17 & 0.48 & 1.47 & 99.52 & 0.48 & 1.46 \\ \hline
2  & 57.20 & 0.70 & 2.28 & 57.25 & 0.70 & 2.23 & 57.65 & 0.70 & 2.27 & 57.14 & 0.70 & 2.17 \\ \hline
3  & 16.51 & 0.91 & 2.44 & 16.50 & 0.91 & 2.44 & 16.94 & 0.92 & 2.43 & 16.49 & 0.91 & 2.44 \\ \hline
4  & 7.32 & 0.96 & 2.52 & 7.31 & 0.96 & 2.50 & 7.78 & 0.96 & 2.55 & 7.30 & 0.96 & 2.50 \\ \hline
5  & 3.28 & 0.98 & 2.68 & 3.28 & 0.98 & 2.68 & 3.72 & 0.98 & 2.66 & 3.28 & 0.98 & 2.68 \\ \hline
6  & 2.47 & 0.99 & 2.73 & 2.47 & 0.99 & 2.74 & 2.92 & 0.99 & 2.72 & 2.47 & 0.99 & 2.73 \\ \hline
7  & 1.82 & 0.99 & 2.73 & 1.82 & 0.99 & 2.74 & 2.27 & 0.99 & 2.75 & 1.82 & 0.99 & 2.74 \\ \hline
8  & 1.54 & 0.99 & 3.18 & 1.53 & 0.99 & 3.26 & 1.99 & 0.99 & 2.98 & 1.53 & 0.99 & 3.24 \\ \hline
9  & 1.31 & 0.99 & 3.22 & 1.29 & 0.99 & 3.27 & 1.74 & 0.99 & 3.03 & 1.29 & 0.99 & 3.24 \\ \hline
10 & 1.09 & 0.99 & 3.56 & 1.06 & 0.99 & 3.36 & 1.51 & 1.00 & 3.05 & 1.06 & 0.99 & 3.32 \\ \hline
11 & 0.88 & 1.00 & 3.68 & 0.85 & 1.00 & 3.68 & 1.30 & 1.00 & 3.69 & 0.85 & 1.00 & 3.68 \\ \hline
12 & 0.71 & 1.00 & 3.68 & 0.68 & 1.00 & 3.68 & 1.14 & 1.00 & 3.63 & 0.68 & 1.00 & 3.68 \\ \hline
13 & 0.54 & 1.00 & 3.68 & 0.53 & 1.00 & 3.69 & 0.97 & 1.00 & 3.67 & 0.53 & 1.00 & 3.69 \\ \hline
14 & 0.42 & 1.00 & 3.68 & 0.40 & 1.00 & 3.69 & 0.85 & 1.00 & 3.67 & 0.40 & 1.00 & 3.69 \\ \hline
15 & 0.33 & 1.00 & 3.69 & 0.31 & 1.00 & 3.69 & 0.75 & 1.00 & 3.68 & 0.31 & 1.00 & 3.69 \\ \hline
16 & 0.25 & 1.00 & 3.69 & 0.23 & 1.00 & 3.69 & 0.68 & 1.00 & 3.72 & 0.23 & 1.00 & 3.69 \\ \hline
17 & 0.18 & 1.00 & 3.69 & 0.16 & 1.00 & 3.69 & 0.61 & 1.00 & 3.79 & 0.16 & 1.00 & 3.69 \\ \hline
18 & 0.13 & 1.00 & 3.69 & 0.11 & 1.00 & 3.69 & 0.55 & 1.00 & 3.72 & 0.11 & 1.00 & 3.69 \\ \hline
19 & 0.09 & 1.00 & 3.69 & 0.07 & 1.00 & 3.69 & 0.52 & 1.00 & 3.75 & 0.07 & 1.00 & 3.69 \\ \hline
\end{tabular}
\label{tab:pca_comparison}
\end{table}

\begin{table}[htbp]
\centering
\scriptsize
\renewcommand{\arraystretch}{1.4}
\setlength{\tabcolsep}{3pt} 
\caption{Comparison of PCA Methods across Different Dimensions for Obesity Dataset} 
\begin{tabular}{|c|ccc|ccc|ccc|ccc|}
\hline
\multirow{2}{*}{\textbf{Dims}} & \multicolumn{3}{c|}{\textbf{covFairPCA}} & \multicolumn{3}{c|}{\textbf{FairPCA\_SDP\_LP}} & \multicolumn{3}{c|}{\textbf{JevdPCA}} & \multicolumn{3}{c|}{\textbf{StandardPCA}} \\
\cline{2-13}
 & \textbf{R.E.} & \textbf{V.E.} &  $\textbf{MMD}^2$ & \textbf{R.E.} & \textbf{V.E.} &  $\textbf{MMD}^2$ & \textbf{R.E.} & \textbf{V.E.} & $\textbf{MMD}^2$ & \textbf{R.E.} & \textbf{V.E.} &  $\textbf{MMD}^2$\\
\hline
1 & 43.67 & 0.94 & 72.12 & 43.72 & 0.94 & 72.16 & 44.84 & 0.94 & 71.80 & 43.67 & 0.94 & 72.12 \\
2 & 4.96 & 0.99 & 72.16 & 4.97 & 0.99 & 72.17 & 5.04 & 0.99 & 72.08 & 4.96 & 0.99 & 72.16 \\
3 & 4.10 & 0.99 & 73.33 & 4.11 & 0.99 & 73.20 & 4.19 & 0.99 & 73.49 & 4.10 & 0.99 & 73.23 \\
4 & 3.48 & 1.00 & 73.66 & 3.49 & 1.00 & 73.63 & 3.57 & 1.00 & 73.72 & 3.48 & 1.00 & 73.68 \\
5 & 2.95 & 1.00 & 74.35 & 2.96 & 1.00 & 74.38 & 3.04 & 1.00 & 74.26 & 2.95 & 1.00 & 74.36 \\
6 & 2.53 & 1.00 & 74.46 & 2.53 & 1.00 & 74.47 & 2.67 & 1.00 & 74.27 & 2.53 & 1.00 & 74.46 \\
7 & 2.17 & 1.00 & 74.47 & 2.17 & 1.00 & 74.47 & 2.32 & 1.00 & 74.51 & 2.17 & 1.00 & 74.46 \\
8 & 1.84 & 1.00 & 74.48 & 1.85 & 1.00 & 74.48 & 1.98 & 1.00 & 74.52 & 1.84 & 1.00 & 74.48 \\
9 & 1.59 & 1.00 & 74.52 & 1.59 & 1.00 & 74.51 & 1.70 & 1.00 & 74.54 & 1.59 & 1.00 & 74.52 \\
10 & 1.37 & 1.00 & 74.53 & 1.37 & 1.00 & 74.53 & 1.51 & 1.00 & 74.54 & 1.37 & 1.00 & 74.53 \\
11 & 1.18 & 1.00 & 74.53 & 1.18 & 1.00 & 74.53 & 1.33 & 1.00 & 74.54 & 1.18 & 1.00 & 74.53 \\
12 & 0.99 & 1.00 & 74.54 & 0.99 & 1.00 & 74.54 & 1.15 & 1.00 & 74.54 & 0.99 & 1.00 & 74.54 \\
13 & 0.83 & 1.00 & 74.54 & 0.83 & 1.00 & 74.54 & 1.00 & 1.00 & 74.54 & 0.83 & 1.00 & 74.54 \\
14 & 0.69 & 1.00 & 74.54 & 0.69 & 1.00 & 74.54 & 0.84 & 1.00 & 74.55 & 0.69 & 1.00 & 74.54 \\
\hline
\end{tabular}
\label{tab:pca_comparison_obesity}
\end{table}

\begin{table}[htbp]
\centering
\scriptsize
\renewcommand{\arraystretch}{1.4}
\setlength{\tabcolsep}{3pt} 
\caption{Comparison of PCA Methods across Different Dimensions for LSAC Dataset} 
\begin{tabular}{|c|ccc|ccc|ccc|ccc|}
\hline
\multirow{2}{*}{\textbf{Dims}} & \multicolumn{3}{c|}{\textbf{covFairPCA}} & \multicolumn{3}{c|}{\textbf{FairPCA\_SDP\_LP}} & \multicolumn{3}{c|}{\textbf{JevdPCA}} & \multicolumn{3}{c|}{\textbf{StandardPCA}} \\
\cline{2-13}
 & \textbf{R.E.} & \textbf{V.E.} &  $\textbf{MMD}^2$ & \textbf{R.E.} & \textbf{V.E.} &  $\textbf{MMD}^2$& \textbf{R.E.} & \textbf{V.E.} &  $\textbf{MMD}^2$& \textbf{R.E.} & \textbf{V.E.} & $\textbf{MMD}^2$ \\
\hline
1 & 18.67 & 0.61 & 0.87 & 18.65 & 0.61 & 0.87 & 18.63 & 0.61 & 0.84 & 17.64 & 0.63 & 0.83 \\
2 & 3.92 & 0.92 & 0.88 & 3.91 & 0.92 & 0.87 & 4.05 & 0.92 & 0.84 & 3.50 & 0.93 & 0.88 \\
3 & 2.72 & 0.94 & 0.91 & 2.72 & 0.94 & 0.91 & 2.87 & 0.94 & 0.86 & 2.31 & 0.95 & 0.92 \\
4 & 1.78 & 0.96 & 0.96 & 1.73 & 0.96 & 0.95 & 1.87 & 0.96 & 0.89 & 1.34 & 0.97 & 0.95 \\
5 & 0.88 & 0.98 & 0.97 & 0.88 & 0.98 & 0.95 & 1.05 & 0.98 & 0.89 & 0.65 & 0.99 & 0.97 \\
6 & 0.56 & 0.99 & 1.05 & 0.56 & 0.99 & 1.05 & 0.71 & 0.99 & 0.92 & 0.41 & 0.99 & 1.91 \\
7 & 0.32 & 0.99 & 1.89 & 0.32 & 0.99 & 1.91 & 0.69 & 0.99 & 0.92 & 0.27 & 0.99 & 1.93 \\
8 & 0.22 & 1.00 & 1.91 & 0.21 & 1.00 & 1.92 & 0.59 & 0.99 & 0.92 & 0.19 & 1.00 & 1.93 \\
9 & 0.15 & 1.00 & 1.91 & 0.13 & 1.00 & 1.92 & 0.55 & 0.99 & 0.92 & 0.12 & 1.00 & 1.93 \\
10 & 0.09 & 1.00 & 1.93 & 0.07 & 1.00 & 1.91 & 0.47 & 0.99 & 0.92 & 0.06 & 1.00 & 1.93 \\
11 & 0.06 & 1.00 & 1.93 & 0.02 & 1.00 & 1.93 & 0.40 & 0.99 & 0.92 & 0.02 & 1.00 & 1.93 \\
\hline
\end{tabular}
\label{tab:pca_comparison_lsac}
\end{table}

\begin{table}[htbp]
\centering
\scriptsize
\renewcommand{\arraystretch}{1.4}
\setlength{\tabcolsep}{3pt} 
\caption{Comparison of PCA Methods across Different Dimensions for NPHA Dataset} 
\begin{tabular}{|c|ccc|ccc|ccc|ccc|}
\hline
\multirow{2}{*}{\textbf{Dims}} & \multicolumn{3}{c|}{\textbf{covFairPCA}} & \multicolumn{3}{c|}{\textbf{FairPCA\_SDP\_LP}} & \multicolumn{3}{c|}{\textbf{JevdPCA}} & \multicolumn{3}{c|}{\textbf{StandardPCA}} \\
\cline{2-13}
 & \textbf{R.E.} & \textbf{V.E.} &  $\textbf{MMD}^2$ & \textbf{R.E.} & \textbf{V.E.} & $\textbf{MMD}^2$ & \textbf{R.E.} & \textbf{V.E.} &  $\textbf{MMD}^2$& \textbf{R.E.} & \textbf{V.E.} & $\textbf{MMD}^2$\\
\hline
1 & 5.04 & 0.31 & 0.00 & 5.02 & 0.31 & 0.00 & 5.03 & 0.31 & 0.00 & 5.02 & 0.31 & 0.00 \\
2 & 4.04 & 0.44 & 0.01 & 4.03 & 0.45 & 0.02 & 4.02 & 0.45 & 0.01 & 4.02 & 0.45 & 0.01 \\
3 & 3.10 & 0.57 & 0.01 & 3.09 & 0.57 & 0.02 & 3.10 & 0.57 & 0.01 & 3.09 & 0.57 & 0.02 \\
4 & 2.53 & 0.65 & 0.02 & 2.51 & 0.65 & 0.02 & 2.53 & 0.65 & 0.02 & 2.51 & 0.65 & 0.02 \\
5 & 2.09 & 0.71 & 0.25 & 2.07 & 0.72 & 0.09 & 2.49 & 0.66 & 0.02 & 2.07 & 0.72 & 0.09 \\
6 & 1.69 & 0.77 & 0.25 & 1.66 & 0.77 & 0.21 & 2.04 & 0.72 & 0.07 & 1.66 & 0.77 & 0.16 \\
7 & 1.32 & 0.82 & 0.26 & 1.30 & 0.82 & 0.26 & 1.63 & 0.78 & 0.09 & 1.30 & 0.82 & 0.19 \\
8 & 1.00 & 0.86 & 0.26 & 0.97 & 0.87 & 0.24 & 1.27 & 0.83 & 0.10 & 0.97 & 0.87 & 0.26 \\
9 & 0.74 & 0.90 & 0.30 & 0.71 & 0.90 & 0.68 & 1.15 & 0.84 & 0.10 & 0.71 & 0.90 & 0.59 \\
10 & 0.51 & 0.93 & 0.87 & 0.48 & 0.93 & 1.02 & 0.68 & 0.91 & 0.32 & 0.48 & 0.93 & 1.00 \\
\hline
\end{tabular}
\label{tab:pca_comparison_npha}
\end{table}

\end{document}